\documentclass{article}

\PassOptionsToPackage{numbers, compress}{natbib}

\usepackage[preprint]{neurips_2024}

\usepackage[utf8]{inputenc} 
\usepackage[T1]{fontenc}    
\usepackage{hyperref}       
\usepackage{url}            
\usepackage{booktabs}       
\usepackage{amsfonts}       
\usepackage{nicefrac}       
\usepackage{microtype}      
\usepackage{xcolor}         

\usepackage{amsmath}
\usepackage{amssymb}
\usepackage{amsthm}
\usepackage{algorithm}
\usepackage{algorithmic}
\usepackage{multirow}
\usepackage{graphicx}
\usepackage{caption}

\def \e {\epsilon}
\def \x {\mathbf{x}}
\def \OO {\mathcal{O}}
\newtheorem{thm}{Theorem}
\newtheorem{prop}{Proposition}
\newtheorem{cor}[thm]{Corollary}

\title{SimInversion: A Simple Framework for Inversion-Based Text-to-Image Editing}

\author{%
  Qi Qian$^1$\thanks{Corresponding author}\qquad Haiyang Xu$^2$\qquad Ming Yan$^2$\qquad Juhua Hu$^3$ \\
  $^1$Alibaba Group, Bellevue, WA 98004, USA\\
  $^2$Alibaba Group, Hangzhou, China\\
  $^3$School of Engineering and Technology, \\University of Washington, Tacoma, WA 98402, USA\\
  \texttt{\{qi.qian, shuofeng.xhy, ym119608\}@alibaba-inc.com, juhuah@uw.edu} 
}

\begin{document}

\maketitle

\begin{abstract}
  Diffusion models demonstrate impressive image generation performance with text guidance. Inspired by the learning process of diffusion, existing images can be edited according to text by DDIM inversion. However, the vanilla DDIM inversion is not optimized for classifier-free guidance and the accumulated error will result in the undesired performance. While many algorithms are developed to improve the framework of DDIM inversion for editing, in this work, we investigate the approximation error in DDIM inversion and propose to disentangle the guidance scale for the source and target branches to reduce the error while keeping the original framework. Moreover, a better guidance scale (i.e., 0.5) than default settings can be derived theoretically. Experiments on PIE-Bench show that our proposal can improve the performance of DDIM inversion dramatically without sacrificing efficiency.
\end{abstract}

\section{Introduction}
\label{sec:intro}
Diffusion models witness the tremendous success of image generation~\cite{ddpm,stable}. To obtain a synthetic image, those methods first sample a random noise from the Gaussian distribution and then a learned denoising network will refine the sample iteratively to recover a high-quality image from the noise. Moreover, text information can be included as conditions for denoising to enable text-guided image generation~\cite{cfree,stable}.

Due to the impressive performance of image generation, text-guided image editing has attracted much attention recently~\cite{ip2p,textual,p2p}. Unlike the guided generation task that only with the target text condition, image editing aims to obtain a target image from a source image with the target condition. Intuitively, image editing can reuse the generation procedure. However, the knowledge from the source image, e.g., structure, background, etc. should be kept in the target image when editing. Considering that the image is generated from the random noise in diffusion, the noise corresponding to the source image can be applied to obtain the target image for knowledge preservation~\cite{p2p}.

However, DDPM sampling~\cite{ddpm} in diffusion is a stochastic process and it is hard to obtain the original noise that generates the source image. To tackle the challenge, a deterministic DDIM sampling is developed to reserve the generation process~\cite{ddim}. It shares the same training objective as DDPM and can infer the initial noise for images with the model pre-trained by DDPM. With the source noise obtained by DDIM inversion, many image editing methods are developed to leverage the generation process in diffusion.

Among different methods, the dual branch framework is prevalent due to its simple yet effective architecture~\cite{masa,p2p,directinv,styled,np,nt}. Concretely, given the initial noise from the source image, the framework consists of denoising processes for the source image and target image, respectively. At each step, the latent state of different branches will be updated with the corresponding conditions and the information from the source branch, e.g., attention map~\cite{p2p}, can be incorporated into the target branch to preserve the content from the source image. Fig.~\ref{fig:illu} (a) illustrates the procedure of DDIM inversion for image editing~\cite{p2p}.

\begin{figure}[t]
\begin{minipage}{0.46\textwidth}
\centering
\includegraphics[height = 1.4in]{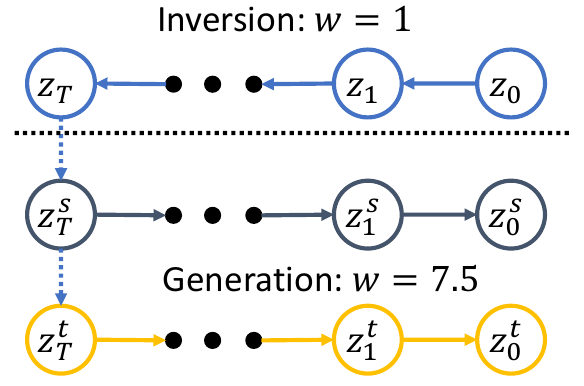}
\mbox{\footnotesize (a) DDIM inversion for editing~\cite{p2p}}
\end{minipage}
\begin{minipage}{0.46\textwidth}
\centering
\includegraphics[height = 1.4in]{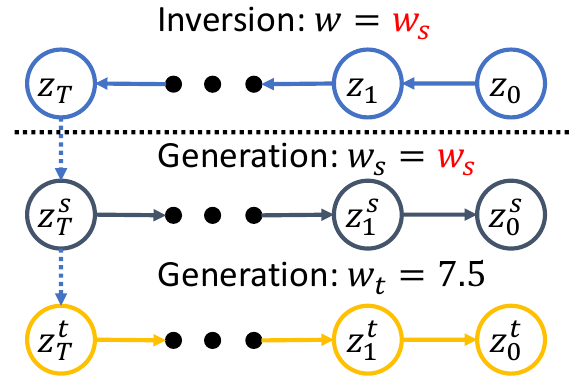}
\mbox{\footnotesize (b) SimInversion for editing}
\end{minipage}
\caption{Illustration of image editing by DDIM inversion and ours. $z_0$ denotes the source image. $z_0^s$ and $z_0^t$ are generated images from the source and target branches, respectively.}\label{fig:illu}
\end{figure}

According to the architecture of the dual branch framework, research efforts can be categorized into two aspects: latent state optimization in the source branch and knowledge transfer to the target branch. First, due to the approximation error in each step of DDIM inversion, the accumulated error will mislead the obtained random noise, which makes it hard to recover the exact source image from the source branch. Many algorithms have been developed to mitigate the challenge. Concretely, negative-prompt inversion~\cite{np} and null-text inversion~\cite{nt} optimize the embeddings of null text in classifier-free guidance to minimize the gap between latent states from inversion and generation while direct inversion~\cite{directinv} directly projects the latent states in the generation process back to those in the inversion process to reduce the approximation error. Second, appropriate knowledge from the source branch will be transferred to the target branch for editing. For the target branch, prompt-to-prompt editing~\cite{p2p} shows that fusing the cross-attention map from the source branch can introduce the desired knowledge to the target branch, and Masactrl~\cite{masa} studies the self-attention layer for effective transfer. 

In this work, we revisit the vanilla DDIM inversion for the source branch. First, we find that the approximation error in the source branch is mainly from the asymmetric guidance scale from the classifier-free guidance, where the inversion process has the guidance scale as $1$ while the generation process holds a much larger scale to focus on the text condition as shown in Fig.~\ref{fig:illu} (a). To mitigate the issue from the asymmetric guidance scale, we propose to keep the same guidance scale for inversion and generation in the source branch. The symmetric guidance scale helps reduce the approximation error in the latent states for the source image. 

Moreover, we investigate the selection of the symmetric guidance scale and our theoretical analysis shows that the approximation error can be further reduced by selecting an appropriate guidance scale beyond the default settings. By adopting the symmetric guidance scale for the source branch while keeping the target branch unchanged, our simple framework improves the performance of vanilla DDIM inversion significantly. Fig.~\ref{fig:illu} (b) illustrates the proposed method. Compared with DDIM inversion, our method only changes the weight of the guidance scale without introducing any additional operations. The main contributions of this work can be summarized as follows.

\begin{itemize}
\item To improve the generation fidelity of the source image, we propose to have a symmetric guidance scale for the source branch while keeping the large guidance scale for the target branch. The method only changes one parameter in DDIM inversion without sacrificing the simplicity of the framework. 
\item We analyze the selection of the guidance scale from the perspective of minimizing the approximation error. While Corollary~\ref{cor:2} shows that the default setting of $0$ or $1$ is a good choice, Corollary~\ref{cor:1} implies that there can be an optimal solution in $[0,1]$, e.g., $0.5$ as suggested by our experiments. 
\item The proposed simple inversion (SimInversion) is evaluated on the recently proposed data set PIE-Bench~\cite{directinv}. Experiments with different editing types confirm that our method can improve the DDIM inversion for image editing under the same framework.
\end{itemize}

\section{Related Work}
\label{sec:related}
While many methods have been developed for image editing~\cite{ip2p,masa,p2p,directinv,styled,np,nt,edict}, we focus on inversion-based methods due to their efficiency and promising performance. Most inversion-based methods have a dual-branch framework to preserve the information from the source branch and edit the image in the target branch. Therefore, different methods are developed for the source branch~\cite{directinv,styled,np,nt} and target branch~\cite{masa,p2p}, respectively. 

For the vanilla DDIM inversion for editing~\cite{p2p}, the source branch will share the large guidance scale from the target branch, which introduces the additional approximation error for recovering the source image and results in the degenerated performance. With the classifier-free guidance, the denoising network depends on a text condition and a null text condition. Therefore, null-text inversion~\cite{nt} proposes to optimize the embeddings of null text to reduce the error from the asymmetric guidance scale. However, the embedding has to be optimized by learning, which is time-consuming. Negative-prompt inversion~\cite{np} improves the learning process of null text embeddings by setting it as the text condition from the source image, which implies a guidance scale of $1$ for the source branch. In this work, we keep the original null text embeddings but change the guidance scale for the source branch directly. Moreover, our analysis shows that $1$ is not the optimal solution for minimizing the approximation error and the performance can be facilitated with an appropriate guidance scale, e.g., $0.5$. Recently, direct inversion~\cite{directinv} studies an extreme case that projects the latent states of the generation process to those from the inversion process for the source branch to eliminate the accumulated errors from the iterative steps. While this method demonstrates better performance than \cite{nt} and \cite{np}, it requires an additional DDIM forward pass for projection, which increases the inference time. Moreover, the projection operator interrupts the generation process of the source branch, which may degenerate the editing performance. On the contrary, our method keeps the simple framework of the DDIM inversion and shows that the approximation error can be minimized by obtaining an appropriate guidance scale for the source branch. Note that some work (e.g., \cite{edict}) tries to eliminate approximation error with a more complex framework that may reduce the editability as shown in~\cite{directinv}. Therefore, only methods sharing a similar framework will be compared in the empirical study.

\section{Simple Inversion}
\label{sec:method}
Before introducing our method, we will briefly review DDIM inversion in the next subsection.

\subsection{DDIM Inversion for Image Editing}
Diffusion models propose to generate an image from a random sampled noise by a series of denoising steps. Concretely, given the noise $z_T$, an image can be obtained iteratively with $t$ from $T$ to $1$ according to
\begin{eqnarray}\label{eq:ddim}
z_{t-1} = \sqrt{\frac{\alpha_{t-1}}{\alpha_t}}z_t + \sqrt{\alpha_{t-1}}(\sqrt{\frac{1}{\alpha_{t-1}}-1} - \sqrt{\frac{1}{\alpha_t}-1})\e_\theta(z_t,t)
\end{eqnarray}
where $z_0$ denotes the generated image or its corresponding latent for the decoder. $\e_\theta(z_t,t)$ is a learned model to predict the noise at $t$-th iteration, where a text condition $C$ can be included for classifier-free guidance~\cite{cfree} as $\e_\theta(z_t,t, C)$ or a null text condition as $\e_\theta(z_t,t, \emptyset)$. $\{\alpha_t\}$ is a sequence of predefined constants for denoising. The process is known as DDIM sampling~\cite{ddim}, which is a deterministic sampling but shares the same training objective as DDPM~\cite{ddpm}.

Given a source image $z_0^s$ and its corresponding text condition $C_s$ (e.g., caption), image editing with text guidance aims to obtain a new image $z_0^t$ with the target text condition $C_t$. According to Eqn.~\ref{eq:ddim}, the target image will be obtained from a random noise $z_T^t$. To preserve the knowledge from the source image, the random noise $z_T^s$ that generates the source image will be applied as $z_T^t=z_T^s$ for image editing.

However, the initial noise $z_T^s$ cannot be inferred from Eqn.~\ref{eq:ddim} with the image $z_0^s$. To illustrate the issue, we rearrange the terms in Eqn.~\ref{eq:ddim} and have
\begin{eqnarray}\label{eq:ori}
z_t = \sqrt{\frac{\alpha_t}{\alpha_{t-1}}} z_{t-1} +\sqrt{\alpha_t}(\sqrt{\frac{1}{\alpha_t}-1} - \sqrt{\frac{1}{\alpha_{t-1}}-1})\e_\theta(z_t,t)
\end{eqnarray}
where the estimation of $z_t$ depends on the prediction from itself $\e_\theta(z_t,t)$. To approximate the reverse of the generation process, DDIM inversion~\cite{ddim} considers replacing $z_t$ by $z_{t-1}$ for denoising and the process becomes
\begin{eqnarray}\label{eq:ddiminv}
z_t = \sqrt{\frac{\alpha_t}{\alpha_{t-1}}} z_{t-1} +\sqrt{\alpha_t}(\sqrt{\frac{1}{\alpha_t}-1} - \sqrt{\frac{1}{\alpha_{t-1}}-1})\e_\theta(z_{t-1},t)
\end{eqnarray}
which helps obtain the initial noise $z_T^s$ for editing. With $z_T^s$, many existing methods apply the dual branch framework for editing~\cite{p2p,directinv,np,nt}, where one branch is for recovering the source image with $C_s$ and the other is to encode the target information with $C_t$. By fusing the generation processes of these branches, the source image can be edited according to the target text condition.

\subsection{Approximation Error in Dual Branch Image Editing}
Since the generation performance heavily depends on the inversion process, we will investigate the approximation error in DDIM inversion to demonstrate our motivation. In Eqn.~\ref{eq:ddiminv}, $\e_\theta(z_t,t)$ is approximated by $\e_\theta(z_{t-1},t)$ with the assumption that $z_{t-1}$ is close to $z_t$ in \cite{ddim}. To analyze the error, we apply the Taylor expansion for vector-valued function on $z_{t-1}$ and rewrite the denoising network on $z_t$ as
\begin{eqnarray}
h(\e_\theta(z_t,t)) \approx h(\e_\theta(z_{t-1},t))+ J_\e(z_{t-1})h(z_t -z_{t-1}) + o(\|h(z_t-z_{t-1})\|_2)
\end{eqnarray}
where $J_\e(z_{t-1})$ is the Jacobian matrix of $\e$ on $z_{t-1}$. $h$ is a reshape operator that converts the tensor $z$ to a vector. Then, the approximation error can be depicted in the following proposition. All detailed proof of this work can be found in the appendix.

\begin{prop}\label{prop:1}
Assuming that the gradient of $\e$ on $z_{t-1}$ is bounded as $\|J_\e(z_{t-1})\|_F\leq c$, we have
\begin{eqnarray}
\|h(\e_\theta(z_t,t)) - h(\e_\theta(z_{t-1},t))\|_2\leq \OO(\|h(z_t-z_{t-1})\|_2)
\end{eqnarray}
\end{prop}

According to Proposition~\ref{prop:1}, when $\|h(z_t-z_{t-1})\|_2$ is sufficiently small, the approximation error $\delta=\|h(\e_\theta(z_t,t)) - h(\e_\theta(z_{t-1},t))\|_2$ becomes negligible, which is consistency with the observation in \cite{ddim}.

However, to improve the sample quality with the text condition, a different denoising step is adopted for generation, which amplifies the error. Concretely, the combined prediction is applied for classifier-free guidance~\cite{cfree}
\begin{eqnarray}\label{eq:w}
\e'(z_t,t,w) = w\e(z_t, t, C) + (1-w)\e(z_t,t, \emptyset)
\end{eqnarray}
where $w$ is the guidance scale and $w>1$ is to emphasize the text condition $C$. The asymmetric process for image editing with vanilla DDIM inversion is illustrated in Fig.~\ref{fig:illu} (a).

\subsection{Simple Inversion with Symmetric Guidance Scale}

In DDIM inversion, both the source branch and target branch share the same $w$ for generation. While the target branch is for editing that focuses on the sample quality of following the text prompt, the source branch is to preserve the information from the source image and the symmetric guidance scale is essential for reducing the approximation error as in the following proposition.
\begin{prop}\label{prop:2}
Let $w_i$ and $w_g$ denote the guidance scale for the inversion and generation process, respectively. Then, with Eqn.~\ref{eq:ori}, the source image can be recovered perfectly when $w_i=w_g$.
\end{prop}

Therefore, we propose to disentangle the guidance scale between dual branches and adopt the symmetric guidance scale for the source branch. 

First, to keep the editability, the large $w_t$ for the target branch remains the same for generation. Unlike DDIM inversion where $w_t$ is also applied for the source branch, we have a different $w_s$ instead. 

Considering that the source branch is for knowledge preservation, we have the same $w_s$ in the inversion and generation for the source branch to eliminate the approximation error from the asymmetric process as suggested by Proposition~\ref{prop:2}. The process of proposed simple inversion can be found in Fig.~\ref{fig:illu} (b).

Finally, we find that the value of $w_s$ can be further optimized to minimize the approximation error as follows.

\begin{prop}\label{prop:3}
Let $\x_c = h(\e(z_t, t, C) - \e(z_{t-1}, t, C))$ and $\x_\emptyset = h(\e(z_t,t, \emptyset)-\e(z_{t-1},t, \emptyset))$, then $\delta(w) = \|h(\e'_\theta(z_t,t,w)) - h(\e'_\theta(z_{t-1},t,w))\|_2$ is minimized when
\begin{eqnarray}
w^* = (\x_\emptyset-\x_c)^\top \x_\emptyset / \|\x_c-\x_\emptyset\|_2^2
\end{eqnarray}
\end{prop}

Proposition~\ref{prop:3} indicates that there is an optimal guidance scale that can minimize the approximation error. Moreover, its scale can be further demonstrated in the following Corollary. 

\begin{cor}\label{cor:1}
With notations in Proposition~\ref{prop:2}, if assuming the approximation error is independent between text condition and null text condition as $\x_c^\top\x_\emptyset=0$, we have $|w^*|\leq 1$.
\end{cor}

Corollary~\ref{cor:1} shows the possible range for the optimal $w_s$, but it is still challenging to obtain the result without $\e(z_t,t,C)$ and $\e(z_t,t,\emptyset)$. To set an applicable weight, we investigate the upper-bound of the error.

\begin{prop}
With notations in Proposition~\ref{prop:2}, we have
\begin{eqnarray}
\delta(w)\leq |w|\|\x_c\|_2 + |1-w| \|\x_\emptyset\|_2
\end{eqnarray}
\end{prop}
\begin{proof}
It is directly from the triangle inequality.
\end{proof}
Since $w\geq 0$ for text condition, the upper-bound can be minimized as
\begin{cor}\label{cor:2}
Let $\delta'(w) = |w|\|\x_c\|_2 + |1-w| \|\x_\emptyset\|_2$. When $w\geq 0$, we have
\begin{eqnarray}
\delta'(w)\geq \min\{\delta'(0),\delta'(1)\}
\end{eqnarray}
\end{cor}
\begin{proof}
It is due to that the $\delta'(w)$ is a monotonic function when $w\in[0,1]$ and $w\in[1,\infty)$.
\end{proof}

Corollary~\ref{cor:2} demonstrates that the upper-bound of the approximation error can be minimized when $w=0$ or $1$. While the optimal $w^*$ varies for different images, \{0,1\} guarantees the overall worst-case performance, which is consistent with the training process of DDPM sampling with classifier-free guidance~\cite{cfree}. Therefore, we can empirically set $w_s$ to be $0$ or $1$ and implement the symmetric DDIM inversion in Alg.~\ref{alg:1}. Compared with the vanilla DDIM inversion, the only difference is the symmetric generation process for the source branch as shown in Step~4. Since the framework of DDIM inversion has not been changed, our proposed simple inversion (SimInversion) can be incorporated with existing editing methods to improve the performance of DDIM inversion.

\begin{algorithm}[t]
\caption{\textbf{Sim}ple \textbf{Inversion} for Image Editing (SimInversion)}
\begin{algorithmic}[1]
\STATE {\bf Input:} source image $z_0^s$, source prompt $C_s$, target prompt $C_t$, $w_s$, $w_t$
\STATE Obtain $z_T$ with $w_s$ by DDIM inversion in Eqn.~\ref{eq:ddiminv}
\FOR{$t=T,...,1$}
\STATE Obtain $z_{t-1}^s$ by \textcolor{red}{$w_s$} and $C_s$ // \textcolor{blue}{DDIM Inversion: Obtain $z_{t-1}^s$ by $w_t$ and $C_s$}
\STATE Obtain $z_{t-1}^t$ by $w_t$ and $C_t$
\STATE Edit $z_{t-1}^t$ by $z_{t-1}^s$ with any existing editing methods
\ENDFOR
\RETURN $z_0^t$
\end{algorithmic}\label{alg:1}
\end{algorithm}

\section{Experiments}
\label{sec:exp}

We conduct experiments on PIE-Bench~\cite{directinv} to evaluate the proposed method. The data set contains 700 images from natural and artificial scenes. Each image is associated with one editing type, and there is a total of 10 different editing types. For a fair comparison, 7 metrics from \cite{directinv} are included to measure the quality of image editing. Concretely, the evaluation comprises edit prompt-image consistency for the whole image and editing regions~\cite{clipsim}, structure distance for structure preservation\cite{sdist}, background preservation by PSNR, LPIPS~\cite{lpips}, MSE and SSIM~\cite{ssim}. While edit prompt-image consistency measures editability, all other metrics are for knowledge preservation. More details can be found in \cite{directinv}. All experiments of this work are implemented on a single V100 GPU and the average performance over all images is reported.

\subsection{Quantitive Comparison on PIE-Bench}
We compare the proposed method to existing inversion-based methods in Table~\ref{ta:inv}. Since our method is to improve the source branch, baseline methods focusing on the same problem are included, i.e., Null-Text inversion (NT)~\cite{nt}, Negative-Prompt inversion (NP)~\cite{np}, StyleDiffsuion (SD)~\cite{styled}, Direct inversion (Direct)~\cite{directinv} and the vanilla DDIM inversion (DDIM). Both NT and SD require an additional optimization process while others are in a training-free manner. Meanwhile, prompt-to-prompt editing~\cite{p2p} for the target branch is applied for all methods due to its superior performance over other methods~\cite{directinv}. In addition, stable diffusion v1.4~\cite{stable} is adopted as the diffusion model shared by all methods. The steps in inversion and generation are set to $50$. For our method, we have ``SimInv-w'' denoting the different guidance scales of $w_s$, and $w_t$ is fixed as $7.5$. 
Finally, ``SimInv$^*$'' selects the better performance between ``SimInv-0'' and ``SimInv-1'' for each image.

\begin{table}[htbp]
\centering
\caption{Comparision on PIE-Bench with 7 metrics. All methods have prompt-to-prompt~\cite{p2p} for editing the target branch. ``Opt'' denotes the additional optimization process, which is time-consuming. Results of baseline methods are reported from \cite{directinv}. The best performance is in bold.}\label{ta:inv}
\resizebox{\linewidth}{!}{
\begin{tabular}{|l|c|cc|c|cccc|}\hline
\multirow{3}{*}{Methods}&\multirow{3}{*}{Opt}&\multicolumn{2}{c|}{Editability}&\multicolumn{5}{c|}{Preservation}\\ \cline{3-9}
&&\multicolumn{2}{c|}{CLIP Similariy} & Structure & \multicolumn{4}{c|}{Background Preservation}\\
&&Whole $\uparrow$ &Edited $\uparrow$ &Distance$_{10^{-3}}$ $\downarrow$ &PSNR $\uparrow$     & LPIPS$_{10^{-3}}$ $\downarrow$  & MSE$_{10^{-4}}$ $\downarrow$     & SSIM$_{10^{-2}}$ $\uparrow$\\\hline
DDIM&&25.01       &  22.44 &69.43          & 17.87  & 208.80  & 219.88  & 71.14 \\
NT& \checkmark&24.75        & 21.86&13.44        & 27.03  & 60.67  & 35.86  & 84.11\\
SD&\checkmark&24.78 & 21.72 &\textbf{11.65} & 26.05 & 66.10 & 38.63 & 83.42 \\
NP&&24.61     &    21.87&16.17       &  26.21 &  69.01  &  39.73 & 83.40    \\
Direct&&25.02&22.10&\textbf{11.65}  & \textbf{27.22} & \textbf{54.55}  & \textbf{32.86} & \textbf{84.76} \\\hline
SimInv-0&&24.90&22.19&20.00&24.93&84.68&50.15&81.99\\
SimInv-0.5&&25.18&22.20&15.75&25.89&70.23&40.94&83.42\\
SimInv-1&&25.20&22.21&16.05&25.75&74.29&43.08&83.07\\\hline
SimInv$^*$&&\textbf{25.67}&\textbf{22.80}&14.76&26.04&69.83&40.27&83.51\\\hline
\end{tabular}}
\end{table}

First, we can find that by disentangling the guidance scale for the source branch and target branch, our method improves DDIM on metrics for structure and background preservation by a large margin. It is because a symmetric guidance scale for the source branch has less approximation error than the asymmetric guidance scale, which helps preserve the knowledge from the source image as in Proposition~\ref{prop:2}. With the decoupled guidance scale, the proposed method reduces the structure distance by $78.7\%$, which shows the potential of the simple framework of DDIM inversion. Second, SimInversion achieves the best edit consistency result measured by CLIP similarity. It implies that the image obtained from our method is more consistent with the target text condition than other methods. Since text-guided image generation has been studied extensively, the unchanged target branch in our framework helps unleash the power of the pre-trained image generation model. When comparing our proposal with different source guidance scales, we observe that SimInv-1 outperforms SimInv-0 on all metrics. The binary guidance scale selects different text conditions for denoising. The better performance of SimInv-1 shows that the text condition can be more helpful than the null text condition in reducing the approximation error in inversion. However, if selecting the better performance from SimInversion with $w_s\in\{0,1\}$, the performance of SimInv$^*$ can be further improved. Concretely, the whole image edit consistency increases by $0.47$ while the structure distance decreases by $1.29$. This phenomenon demonstrates that counterintuitively, $w_s=1$ is not always the best option and $w_s=0$ can be more appropriate for some cases as shown in Fig.~\ref{fig:w}. The result further confirms our analysis in Corollary~\ref{cor:2}. Finally, when setting $w_s=0.5$, the preservation performance is consistently improved over $0$ or $1$. This phenomenon verifies our analysis in Corollary~\ref{cor:1} and also suggests a new default value for inversion-based image editing methods. Experiments with other editing methods can be found in the appendix.

\subsection{Qualitative Comparison}
After the quantitive evaluation, we include the edited images for a qualitative comparison in this subsection. Considering that PIE-Bench contains 10 different editing types, 2 examples are illustrated for each editing type and results are summarized in Fig.~\ref{fig:type0}-\ref{fig:type9}. We have Fig.~\ref{fig:type0}-\ref{fig:type3} in the main manuscript while others are in the appendix. Besides the inversion-based methods, we also include a learning-based method Instruct-P2P~\cite{ip2p} that learns a diffusion model with training data obtained from \cite{p2p} in the comparison.

\begin{figure}[htbp]
\centering
\begin{minipage}{\textwidth}
\centering
\includegraphics[height = 2in]{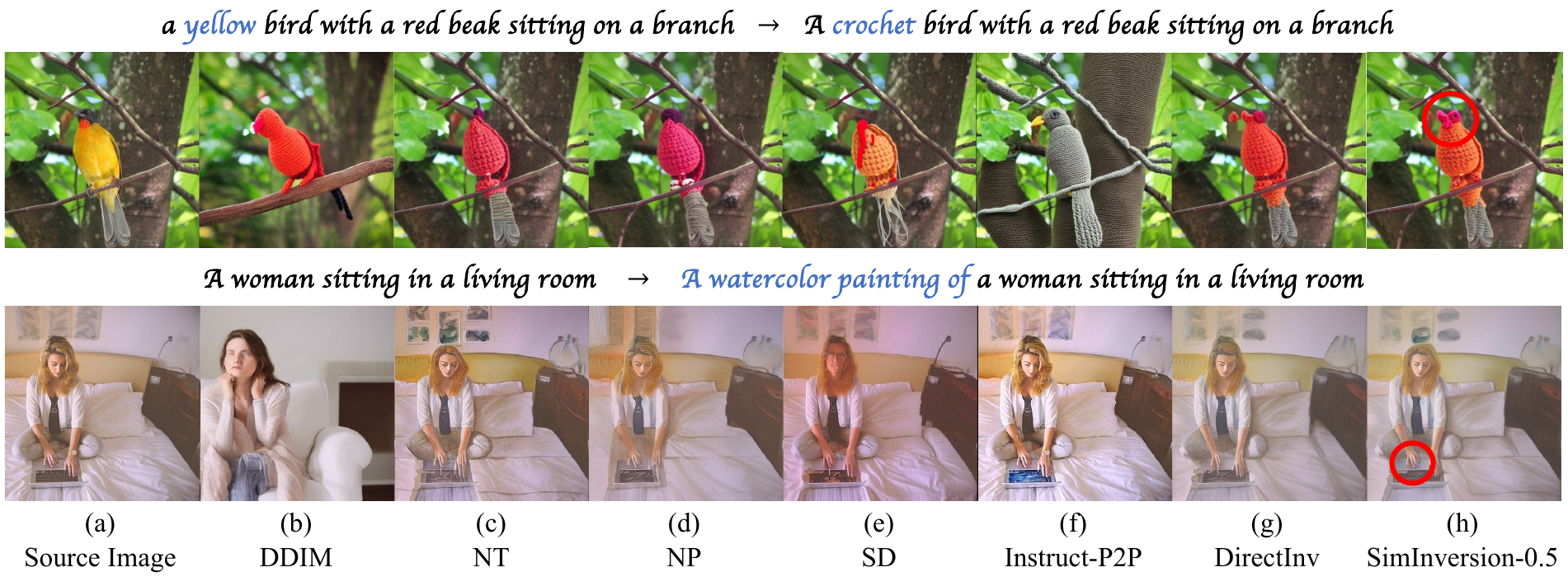}
\captionof{figure}{Illustration of image editing for random editing. The difference is highlighted by red bounding boxes.}\label{fig:type0}
\end{minipage}

\begin{minipage}{\textwidth}
\centering
\includegraphics[height = 2in]{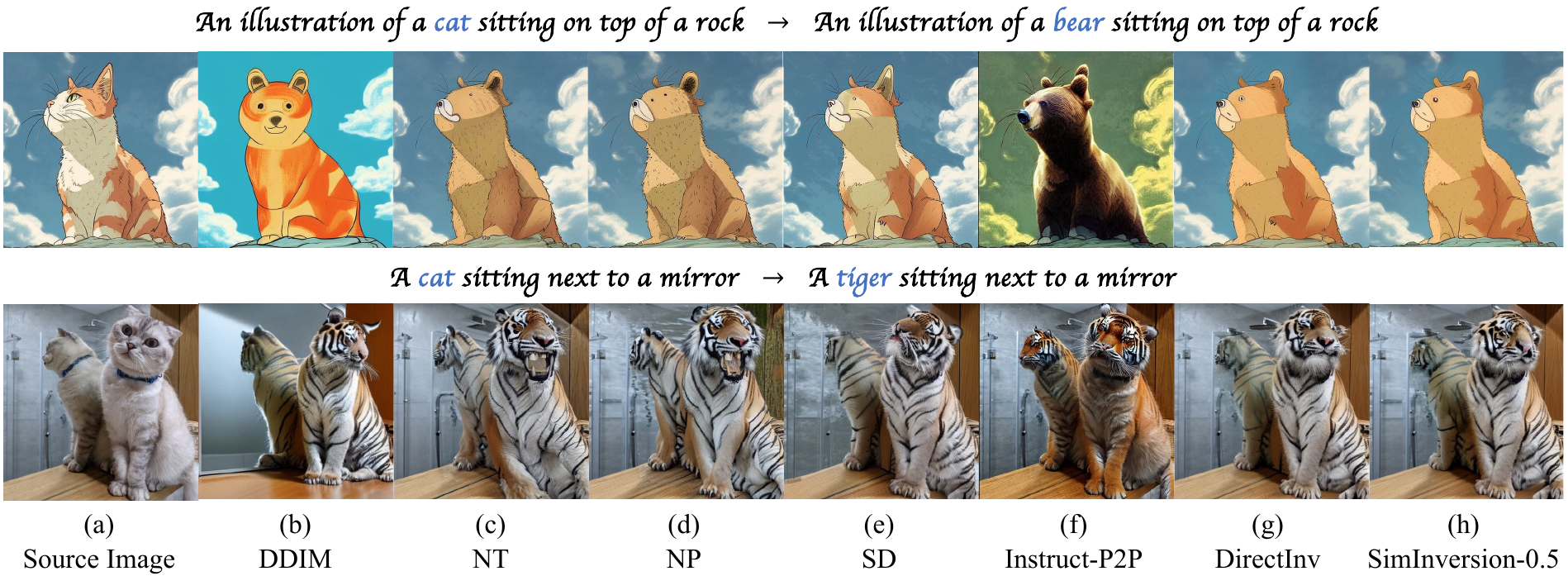}
\captionof{figure}{Illustration of image editing for changing object.}\label{fig:type1}
\end{minipage}

\begin{minipage}{\textwidth}
\centering
\includegraphics[height = 2in]{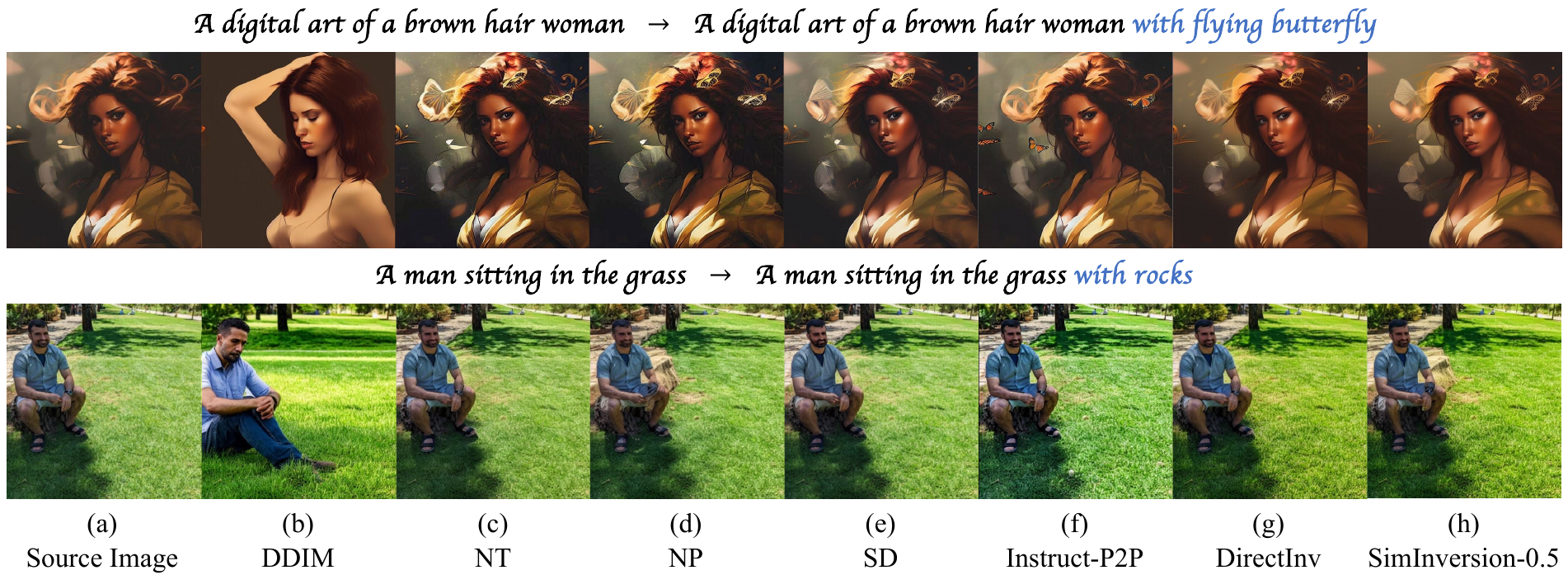}
\captionof{figure}{Illustration of image editing for adding object.}\label{fig:type2}
\end{minipage}
\end{figure}

\begin{figure}[htbp]
\centering
\begin{minipage}{\textwidth}
\centering
\includegraphics[height = 2in]{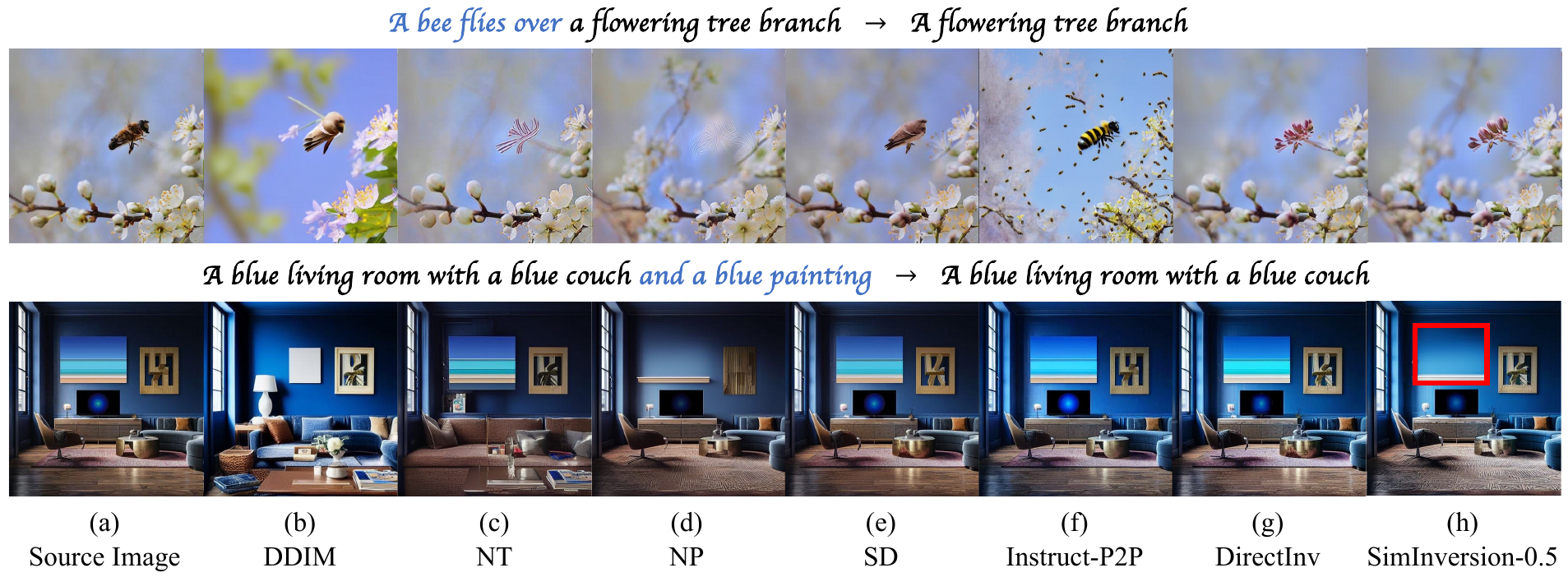}
\captionof{figure}{Illustration of image editing for deleting object. The difference is highlighted by red bounding boxes.}\label{fig:type3}
\end{minipage}
\end{figure}

First, with the simple inversion framework, our method is still effective for editing. We observe that SimInversion can successfully tailor the source image according to the target text prompt for all editing types. Moreover, our method can capture details that are ignored by existing methods. For example, SimInversion preserves the eyes of the bird and the shape of the hand in Fig.~\ref{fig:type0} while direct inversion misses those details. Finally, Instruct-P2P can obtain the target image that is consistent with the target text condition but the details from the source image may be lost. On the contrary, all inversion-based methods can preserve the structure of the source image well. It implies that the inversion-based image editing method can achieve a better trade-off between knowledge preservation and editability.

\subsection{Ablation Study}
\subsubsection{Effect of $w_s$}
While we only have $w_s=\{0,0.5,1\}$ for the main comparison, we vary $w_s$ in $\{0,0.2,0.5,0.8,1,2\}$ and report the performance in Table~\ref{ta:wablation}.

\begin{table}[htbp]
\centering
\caption{Comparision with different $w_s$ for SimInversion on PIE-Bench with 7 metrics. SimInv$^*$ consists of the best performance from $w_s\in\{0,0.2,0.5,0.8,1,2\}$. The best performance excluding SimInv$^*$ is in bold.}\label{ta:wablation}
\resizebox{\linewidth}{!}{
\begin{tabular}{|l|cc|c|cccc|}\hline
\multirow{2}{*}{Methods}&\multicolumn{2}{c|}{CLIP Similariy} & Structure & \multicolumn{4}{c|}{Background Preservation}\\
&Whole $\uparrow$ &Edited $\uparrow$ &Distance$_{10^{-3}}$ $\downarrow$ &PSNR $\uparrow$     & LPIPS$_{10^{-3}}$ $\downarrow$  & MSE$_{10^{-4}}$ $\downarrow$     & SSIM$_{10^{-2}}$ $\uparrow$\\\hline
SimInv-0&24.90&22.19&20.00&24.93&84.68&50.15&81.99\\
SimInv-0.2&25.06&\textbf{22.21}&17.55&25.47&76.46&44.51&82.79\\
SimInv-0.5&25.18&22.20&15.75&\textbf{25.89}&\textbf{70.23}&\textbf{40.94}&\textbf{83.42}\\
SimInv-0.8&25.17&\textbf{22.21}&\textbf{15.66}&25.88&71.58&41.50&83.30\\
SimInv-1&\textbf{25.20}&\textbf{22.21}&16.05&25.75&74.29&43.08&83.07\\
SimInv-2&24.80&22.15&20.52&24.88&87.16&53.15&81.61\\\hline
\end{tabular}}
\end{table}

First, when $w_s\in[0,1]$, the preservation performance of SimInversion surpasses DDIM inversion with a clear margin. According to our analysis in Corollary~\ref{cor:1}, the approximation error can be small in the appropriate range with the mild assumption. Second, while the optimal $w_s$ varies on different metrics, $w_s=0.5$ demonstrates a consistently good performance with all metrics, which is consistent with previous observations. Compared with $w_s\leq 1$, $w_s=2$ performs much worse, especially for the preservation. While the optimal $w_s$ is unbounded, the upper-bound of the approximation error can be optimized by $0,1$ as in Corollary~\ref{cor:2}, which guarantees the worst-case performance. 

\begin{figure}[htbp]
\centering
\includegraphics[height = 5in]{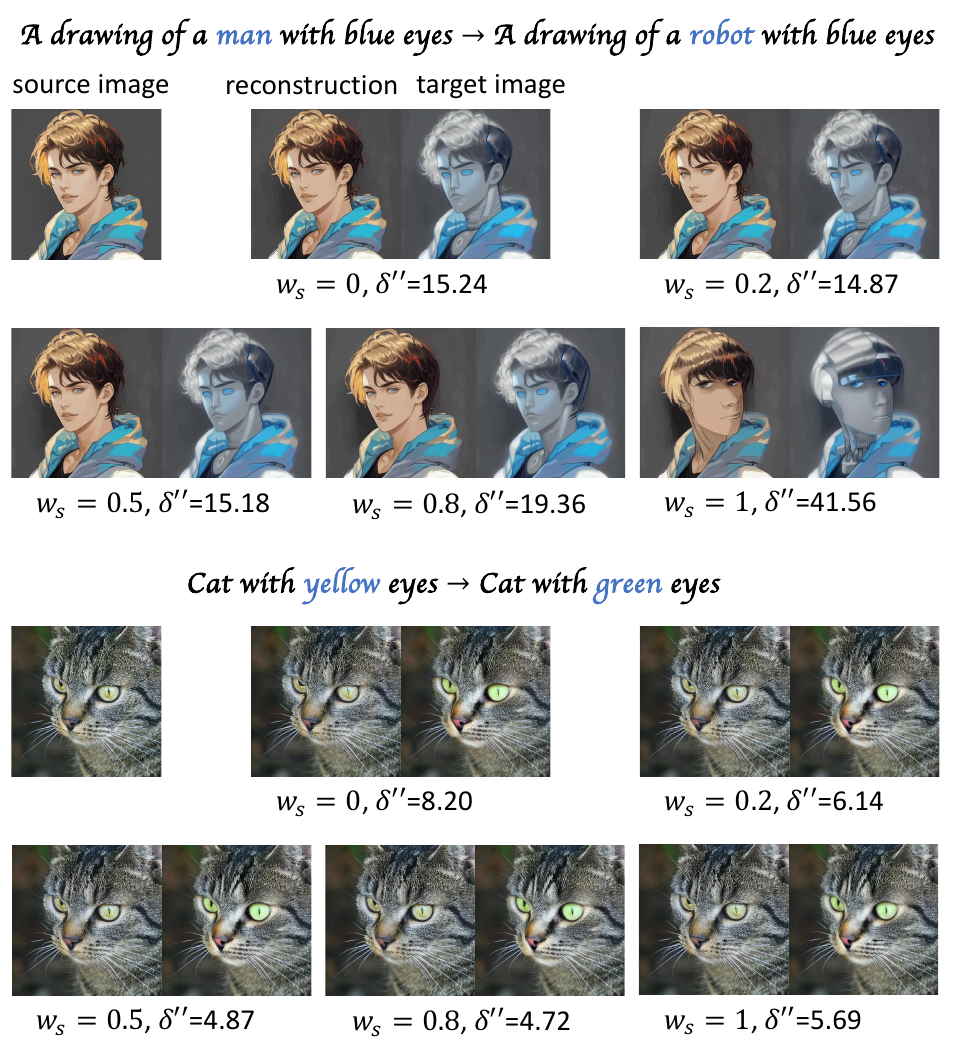}
\captionof{figure}{Illustration of approximation error $\delta''$ with different $w_s$.}\label{fig:w}
\end{figure}

To further illustrate the influence of $w_s$ on reconstruction, Fig.~\ref{fig:w} shows two images edited with different $w_s$ and the approximation error is measured by $\delta'' = \|h(z_0^s - \hat{z}_0^s)\|_2$, where $z_0^s$ denotes the latent state of the source image after encoding and $\hat{z}_0^s$ is the output generated from the source branch. For the first image about a person, $w_s=0$ can recover the source image with a small approximation error. On the contrary, $w_s=1$ fails to capture the source image and the approximation error is more than two times that with $0$. However, $w_s=0.2$ achieves the best reconstruction result, which confirms our analysis. For the second image about the cat, $w_s=1$ outperforms $w_s=0$ while the best performance is from $0.8$. Meanwhile, $w_s=0.5$ is close to the optimal result in different scenarios, which confirms our analysis.

\subsubsection{Comparison of Running Time}
Besides the performance, we compare the running time in Table~\ref{ta:time}. First, all training-free methods, e.g., DDIM, NP, Direct, and ours are much more efficient than the optimization-based methods, e.g., NT and SD. It demonstrates that the cost of optimization overwhelms that of inversion, and is expensive for image editing. Moreover, Direct inversion requires an additional DDIM forward pass, which costs an additional running time of 18s compared with DDIM inversion. On the contrary, when $w_s=\{0,1\}$ as denoted by ``SimInv-binary'', the proposed method only applies the denoising network with a single text condition, which can save the running time on the source branch when generating. Therefore, SimInversion runs even faster than DDIM inversion with the binary guidance scale. When a float guidance scale is adopted as in the ablation study above, e.g., $w_s=0.5$, our inversion process will infer noise from the source text condition and the null text condition simultaneously and thus require some extra time. However, the increase in the running time is only 4s over DDIM inversion, which is still faster than Direct inversion. The comparison shows that our method can improve the performance of DDIM inversion while preserving its efficiency.

\begin{table}[h]
\centering
\caption{Comparision of Running Time for editing a single image with a V100 GPU. SD is slower than NT~\cite{directinv}, which is not included in the comparison.}\label{ta:time}
\begin{tabular}{|c|c|c|c|c|c|c|c|}\hline
Methods&DDIM&NT&NP&Direct&SimInv-binary&SimInv-float\\\hline
Running time (s)&25&171&25&43&21&29\\\hline
\end{tabular}
\end{table}

\section{Conclusion}
\label{sec:conclude}

In this work, we revisit DDIM inversion for image editing. To minimize the approximation error for knowledge preservation, we propose to disentangle the guidance scale for the source and target generation branches and keep a symmetric guidance scale for the source branch. Moreover, our analysis shows that there exists an optimal guidance scale for the source branch which can lie in $[0,1]$. The observation is consistent with the success of DDIM inversion while indicating the future direction for improvement that obtains the appropriate guidance scale for each image efficiently.

\paragraph{Limitations} This work aims to investigate the approximation error in vanilla DDIM inversion. While our performance is competitive, it may not be the best compared to methods with other frameworks.

\paragraph{Broader Impacts} Generated images may provide fake information but it can be mitigated by including watermarks in generated images.

{\small
\bibliographystyle{abbrv}
\bibliography{siminv}
}


\appendix

\section{Theoretical Analysis}
\subsection{Proof of Proposition~\ref{prop:1}}
\begin{proof}
\begin{eqnarray*}
&&\|h(\e_\theta(z_t,t)) - h(\e_\theta(z_{t-1},t))\|_2 = \|J_\e(z_{t-1})h(z_t -z_{t-1})+o(\|h(z_t-z_{t-1})\|_2)\|_2\\
&&\leq \|J_\e(z_{t-1})h(z_t -z_{t-1})\|_2 + o(\|h(z_t-z_{t-1})\|_2)\\
&&\leq c\|h(z_t -z_{t-1})\|_2 + o(\|h(z_t-z_{t-1})\|_2)
\end{eqnarray*}
\end{proof}

\subsection{Proof of Proposition~\ref{prop:2}}
\begin{proof}
We consider the one-step process for analysis.
For inversion, we have
\[z_1 = \sqrt{\frac{\alpha_1}{\alpha_{0}}} z_{0} +\sqrt{\alpha_1}(\sqrt{\frac{1}{\alpha_1}-1} - \sqrt{\frac{1}{\alpha_{0}}-1})\e'_\theta(z_1,1, w_i)\]
where we have $\e'_\theta(z_1,1, w_i)$ as Eqn.~\ref{eq:ori} in lieu of $\e'_\theta(z_0,1, w_i)$ to eliminate the approximation error from noise prediction and focus on the effect of the guidance scale.
Then, the image will be recovered by generation
\[z'_0 = \sqrt{\frac{\alpha_{0}}{\alpha_1}}z_1 + \sqrt{\alpha_{0}}(\sqrt{\frac{1}{\alpha_{0}}-1} - \sqrt{\frac{1}{\alpha_1}-1})\e'_\theta(z_1,1,w_g)\]
The distance to the ground-truth latent state can be computed as
\begin{align*}
&\|z_0-z'_0\| = \|z_0 - \sqrt{\frac{\alpha_{0}}{\alpha_1}}z_1 + \sqrt{\alpha_{0}}(\sqrt{\frac{1}{\alpha_{0}}-1} - \sqrt{\frac{1}{\alpha_1}-1})\e'_\theta(z_1,1,w_g)\|\\
& = \sqrt{\alpha_0}|\sqrt{\frac{1}{\alpha_1}-1} - \sqrt{\frac{1}{\alpha_0}-1}|\|\e'_\theta(z_1,1,w_g) - \e'_\theta(z_1,1,w_i)\|\\
&=\sqrt{\alpha_0}|\sqrt{\frac{1}{\alpha_1}-1} - \sqrt{\frac{1}{\alpha_0}-1}| \|(w_g - w_i)\e(z_1,1,C) + (w_i-w_g)\e(z_1,1,\emptyset)\|
\end{align*}
Therefore, when $w_g=w_i$, the distance is minimized and $z'_0$ recovers the ground-truth $z_0$. The same analysis can be extended to the inversion process with multiple steps.
\end{proof}

\subsection{Proof of Proposition~\ref{prop:3}}
\begin{proof}
With notations in Proposition~\ref{prop:2}, we have
\[\delta(w) = \|w(\x_c-\x_\emptyset) + \x_\emptyset\|_2\]
By minimizing $\delta(w)^2$ and letting the gradient to $0$, we have the desired result.
\end{proof}

\subsection{Proof of Corollary~\ref{cor:1}}
\begin{proof}
The value of $w^*$ can be bounded as
\[\|w^*\|_2 = \|(\x_\emptyset-\x_c)^\top \x_\emptyset\|_2/\|\x_c-\x_\emptyset\|_2^2\leq \|\x_\emptyset\|_2/\|\x_c-\x_\emptyset\|_2\]
With the independent assumption, we have
\[\|w^*\|_2\leq \|\x_\emptyset\|_2/\sqrt{\|\x_c\|_2^2+\|\x_\emptyset\|_2^2}\leq 1\]
\end{proof}

\section{Experiments}

\subsection{Application on other Editing Methods}
Besides prompt-to-prompt editing~\cite{p2p}, we also evaluate the proposed inversion method with MasaCtrl~\cite{masa} in Table.~\ref{ta:masa}. MasaCtrl applies null text for DDIM inversion and source image generation, i.e., $C=\emptyset$, which yields a special case of $w_s=0$ in SimInversion. Nevertheless, we vary the value of $w_s$ and compare the performance to the default setting.

\begin{table}[h]
\centering
\caption{Comparision with different $w_s$ and MasaCtrl for editing. The best performance is in bold.}\label{ta:masa}
\resizebox{\linewidth}{!}{
\begin{tabular}{|c|cc|c|cccc|}\hline
\multirow{2}{*}{Methods}&\multicolumn{2}{c|}{CLIP Similariy} & Structure & \multicolumn{4}{c|}{Background Preservation}\\
&Whole $\uparrow$ &Edited $\uparrow$ &Distance$_{10^{-3}}$ $\downarrow$ &PSNR $\uparrow$     & LPIPS$_{10^{-3}}$ $\downarrow$  & MSE$_{10^{-4}}$ $\downarrow$     & SSIM$_{10^{-2}}$ $\uparrow$\\\hline
DDIM/SimInv-0&23.96& 21.16& 28.38 & 22.17  & 106.62  & 86.97  & 79.67\\ 
SimInv-0.2&24.37&21.36&\textbf{27.23}&\textbf{22.30}&101.31&\textbf{85.55}&80.22\\
SimInv-0.5&24.45&21.38&27.65&22.24&\textbf{99.39}&86.80&\textbf{80.25}\\
SimInv-0.8&24.40&21.44&28.28&22.09&101.56&90.48&79.94\\
SimInv-1.0&\textbf{24.46}&\textbf{21.46}&29.64&21.92&105.61&95.05&79.54\\\hline
\end{tabular}}
\end{table}

According to Table~\ref{ta:masa}, we can observe that $w_s=0$ outperforms $w_s=1$ for MasaCtrl, which confirms the default selection. However, a better preservation and editing performance can be obtained by increasing $w_s$. It is because selecting appropriate $w_s$ can further reduce the approximation error. Finally, $w_s=0.5$ shows a better trade-off between preservation and editability than $w_s=0$. It implies that the choice of $w_s=0.5$ is applicable for different editing methods.

\subsection{Effect of $w_t$}
Besides $w_s$, we also evaluate the effect of different $w_t$ in Table~\ref{ta:wt}. Evidently, $w_t=7.5$ balances the preservation and editability well when $w_s=0.5$. It also confirms that the selection of $w_s$ will not influence that of $w_t$ and these two guidance scales can be disentangled.

\begin{table}[h]
\centering
\caption{Comparision with different $w_t$ for SimInversion on PIE-Bench with 7 metrics. $w_s$ is fixed as $0.5$. The best performance is in bold.}\label{ta:wt}
\resizebox{\linewidth}{!}{
\begin{tabular}{|c|cc|c|cccc|}\hline
\multirow{2}{*}{$w_t$}&\multicolumn{2}{c|}{CLIP Similariy} & Structure & \multicolumn{4}{c|}{Background Preservation}\\
&Whole $\uparrow$ &Edited $\uparrow$ &Distance$_{10^{-3}}$ $\downarrow$ &PSNR $\uparrow$     & LPIPS$_{10^{-3}}$ $\downarrow$  & MSE$_{10^{-4}}$ $\downarrow$     & SSIM$_{10^{-2}}$ $\uparrow$\\\hline
7&25.10&22.15&\textbf{15.21}&\textbf{26.00}&\textbf{69.03}&\textbf{39.92}&\textbf{83.55}\\
7.5&25.18&22.20&15.75&25.89&70.23&40.94&83.42\\
8&\textbf{25.26}&\textbf{22.26}&16.33&25.78&71.46&41.98&83.28\\\hline
\end{tabular}}
\end{table}

\subsection{Effect of Denoising Steps}
To evaluate the performance of $w_s$ with different denoising steps, we vary the number of steps in $\{30,50,100\}$ and summarize the result in Table~\ref{ta:steps}. Obviously, the selection of $w_s$ is robust to the number of steps and $w_s=0.5$ demonstrates the best performance among various settings.

\begin{table}[htbp]
\centering
\caption{Comparision on PIE-Bench with 7 metrics with different numbers of denoising steps.}\label{ta:steps}
\resizebox{\linewidth}{!}{
\begin{tabular}{|l|c|cc|c|cccc|}\hline
\multirow{2}{*}{Methods}&\multirow{2}{*}{\#Steps}&\multicolumn{2}{c|}{CLIP Similariy} & Structure & \multicolumn{4}{c|}{Background Preservation}\\
&&Whole $\uparrow$ &Edited $\uparrow$ &Distance$_{10^{-3}}$ $\downarrow$ &PSNR $\uparrow$     & LPIPS$_{10^{-3}}$ $\downarrow$  & MSE$_{10^{-4}}$ $\downarrow$     & SSIM$_{10^{-2}}$ $\uparrow$\\\hline
SimInv-0&30&24.68&21.89&20.25&24.83&87.02&50.66&81.72\\
SimInv-0.5&30&25.01&22.01&15.90&25.82&72.48&40.98&83.16\\
SimInv-1&30&25.01&22.06&16.57&25.66&75.94&43.06&82.90\\\hline
SimInv-0&50&24.90&22.19&20.00&24.93&84.68&50.15&81.99\\
SimInv-0.5&50&25.18&22.20&15.75&25.89&70.23&40.94&83.42\\
SimInv-1&50&25.20&22.21&16.05&25.75&74.29&43.08&83.07\\\hline
SimInv-0&100&25.04&22.28&18.53&25.30&78.90&47.40&82.63\\
SimInv-0.5&100&25.30&22.31&14.96&26.20&65.96&39.32&83.84\\
SimInv-1&100&25.27&22.30&15.33&26.07&69.43&41.14&83.55\\\hline
\end{tabular}}
\end{table}

Finally, we investigate the gap between $z_t$ and $z_{t-1}$ to evaluate the approximation error from replacing $z_t$ by $z_{t-1}$ in standard DDIM inversion when $w_s=0.5$. The gap is measured by $\ell=E_t[\|h(z_t-z_{t-1})\|_2/\|h(z_{t-1})\|_2]$ and the number of steps varies in $\{10,30,50,100,500\}$. Fig.~\ref{fig:step} shows the difference between $z_t$ and $z_{t-1}$ with image editing. We can observe that the gap is small for approximation and can be further reduced with the increase of denoising steps, which confirms the observation in \cite{ddim}. Moreover, \#steps=50 balance the efficiency and editing performance well.

\begin{figure}[htbp]
\centering
\includegraphics[width = \textwidth]{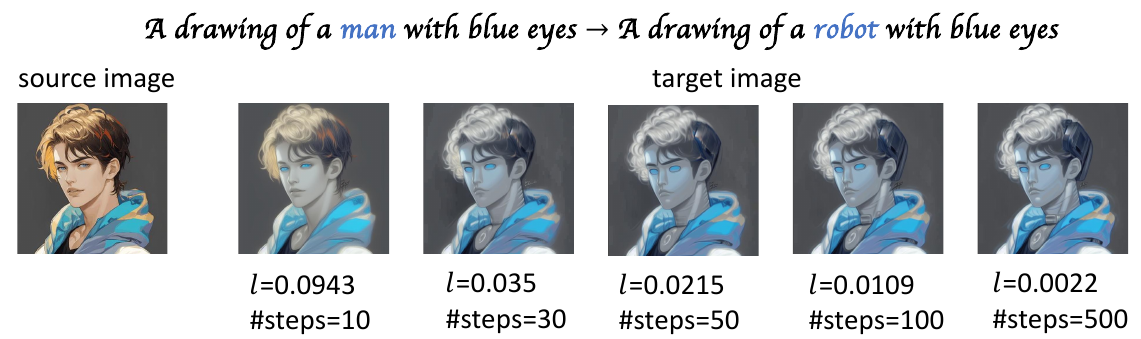}
\captionof{figure}{Illustration of approximation error $\ell=E_t[\|h(z_t-z_{t-1})\|_2/\|h(z_{t-1})\|_2]$ with different numbers of denoising steps.}\label{fig:step}
\end{figure}

\subsection{Qualitative Comparison}

\begin{figure}[htbp]
\centering
\begin{minipage}{\textwidth}
\centering
\includegraphics[height = 2in]{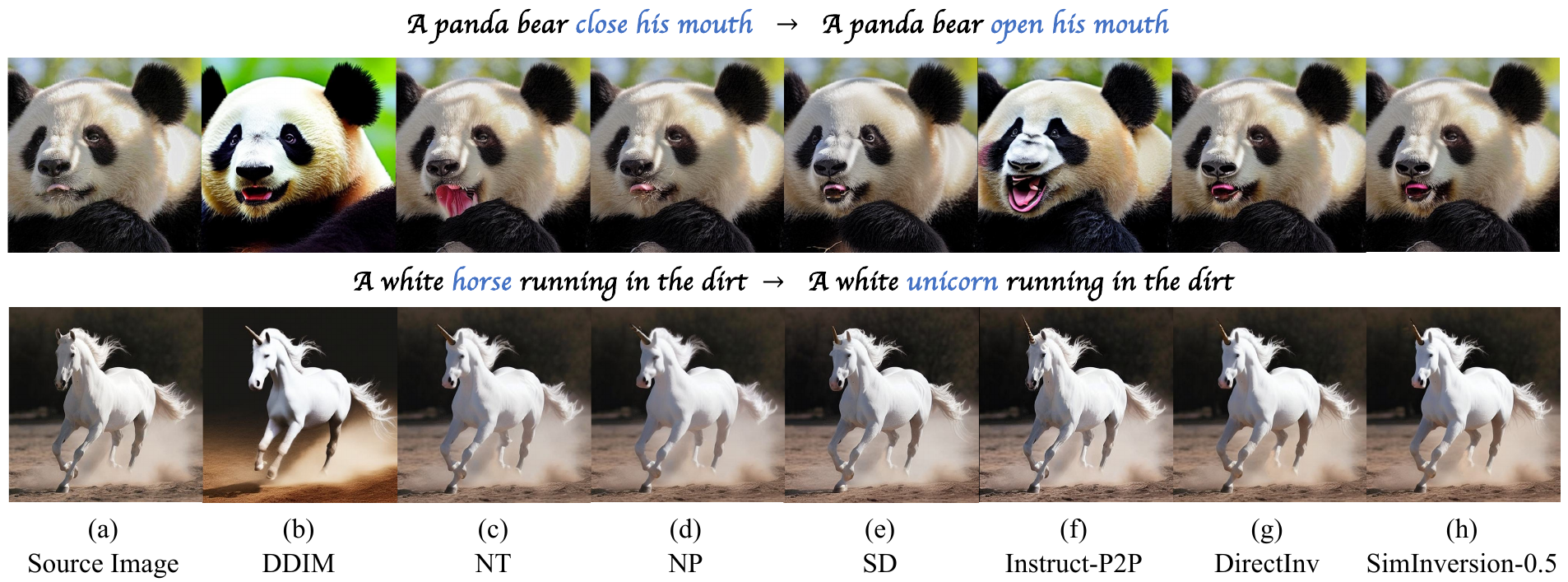}
\captionof{figure}{Illustration of image editing for changing content.}\label{fig:type4}
\end{minipage}

\begin{minipage}{\textwidth}
\centering
\includegraphics[height = 2in]{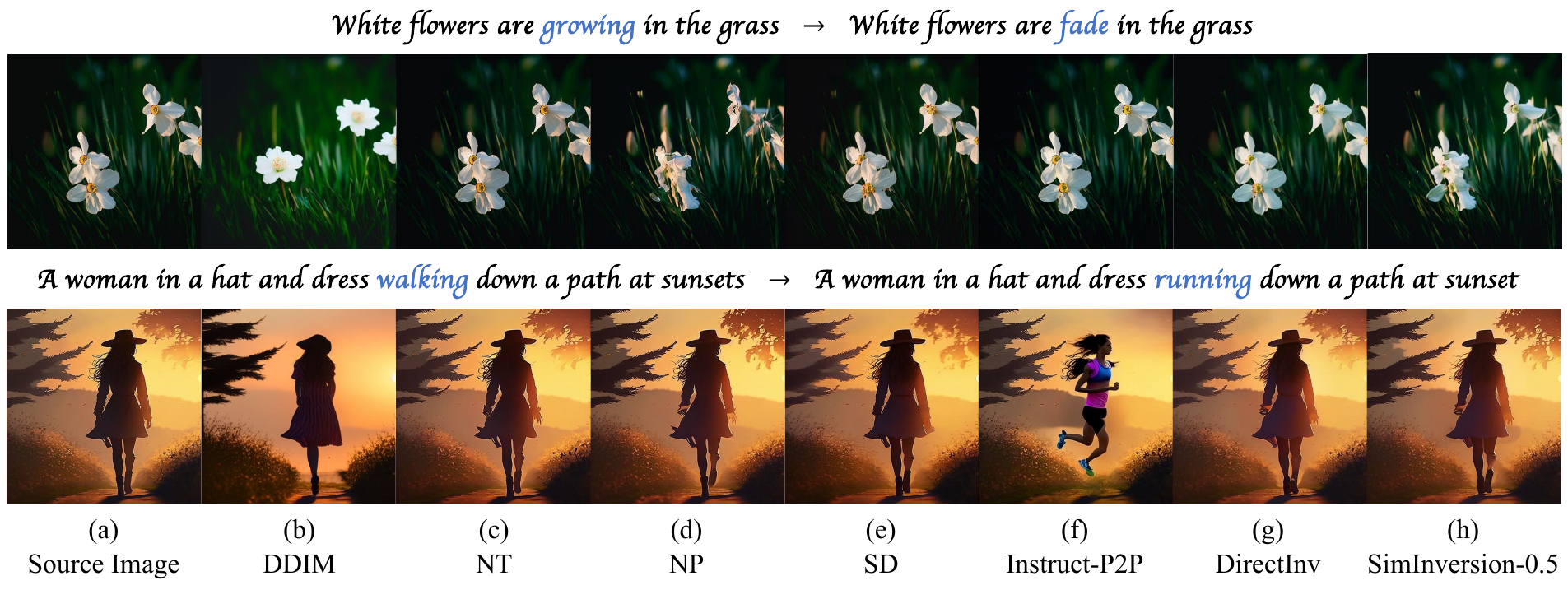}
\captionof{figure}{Illustration of image editing for changing pose.}\label{fig:type5}
\end{minipage}
\end{figure}

\begin{figure}[htbp]
\centering
\begin{minipage}{\textwidth}
\centering
\includegraphics[height = 2in]{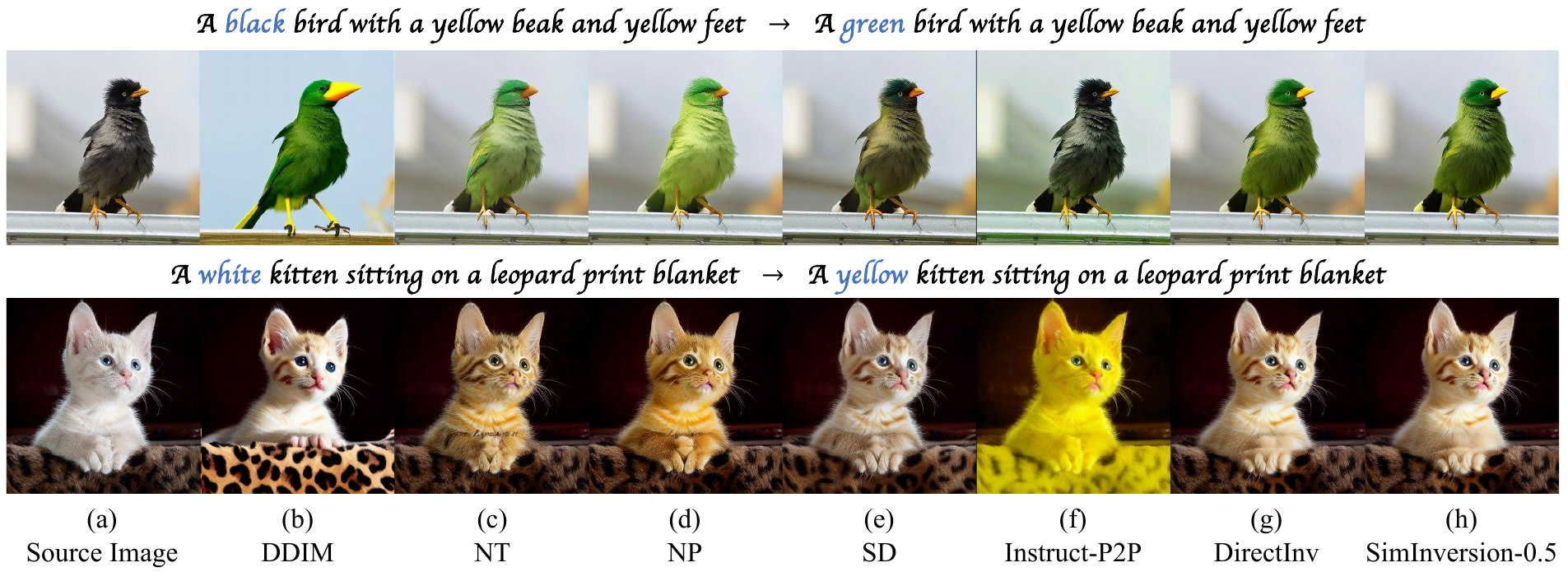}
\captionof{figure}{Illustration of image editing for changing color.}\label{fig:type6}
\end{minipage}

\begin{minipage}{\textwidth}
\centering
\includegraphics[height = 2in]{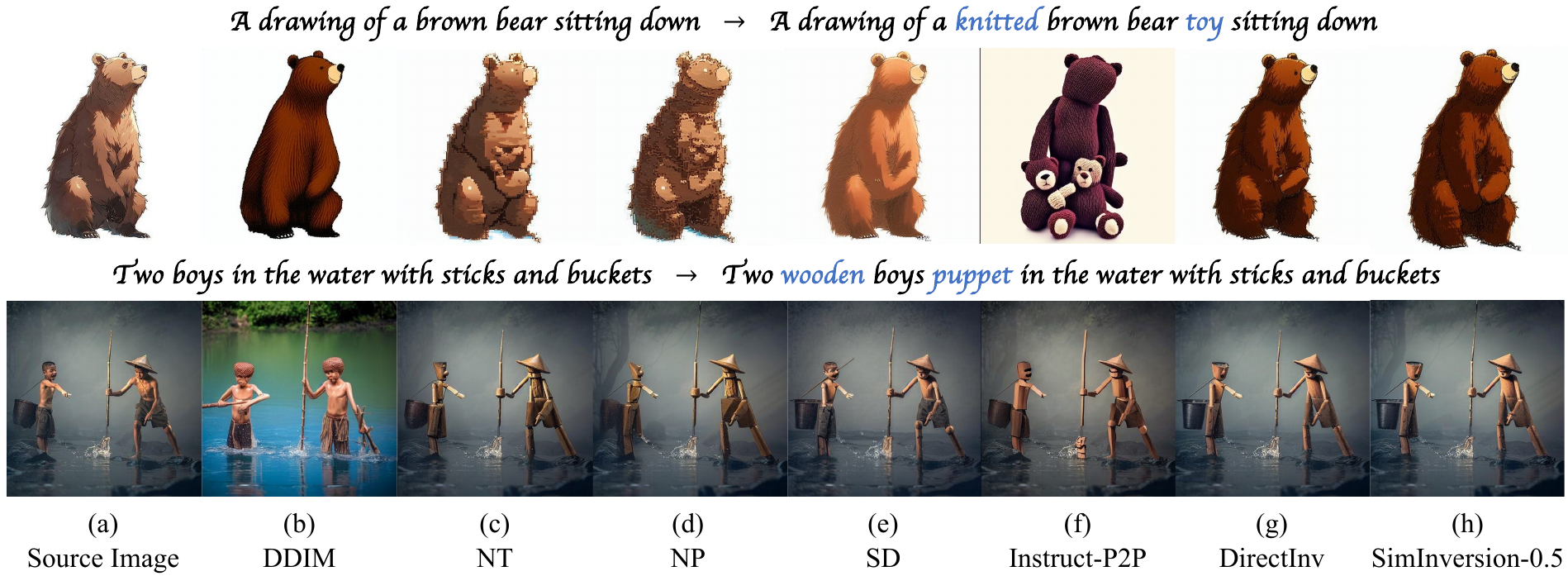}
\captionof{figure}{Illustration of image editing for changing material.}\label{fig:type7}
\end{minipage}

\begin{minipage}{\textwidth}
\centering
\includegraphics[height = 2in]{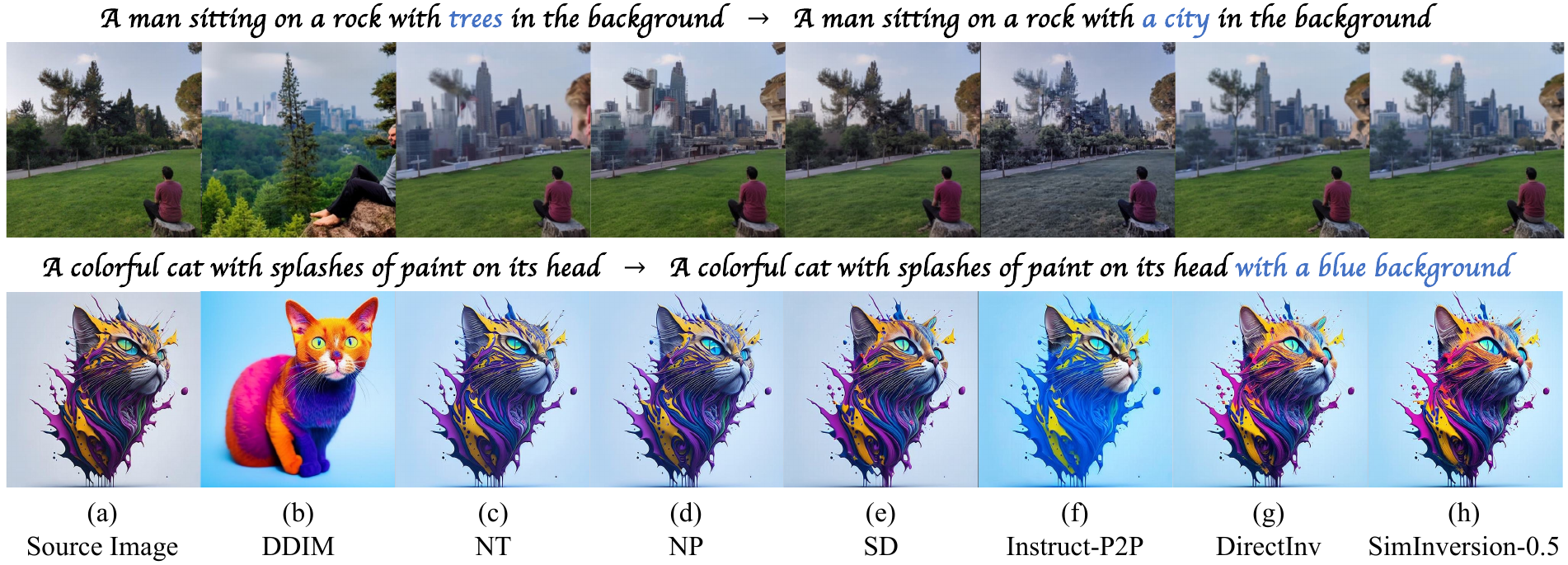}
\captionof{figure}{Illustration of image editing for changing background.}\label{fig:type8}
\end{minipage}

\begin{minipage}{\textwidth}
\centering
\includegraphics[height = 2in]{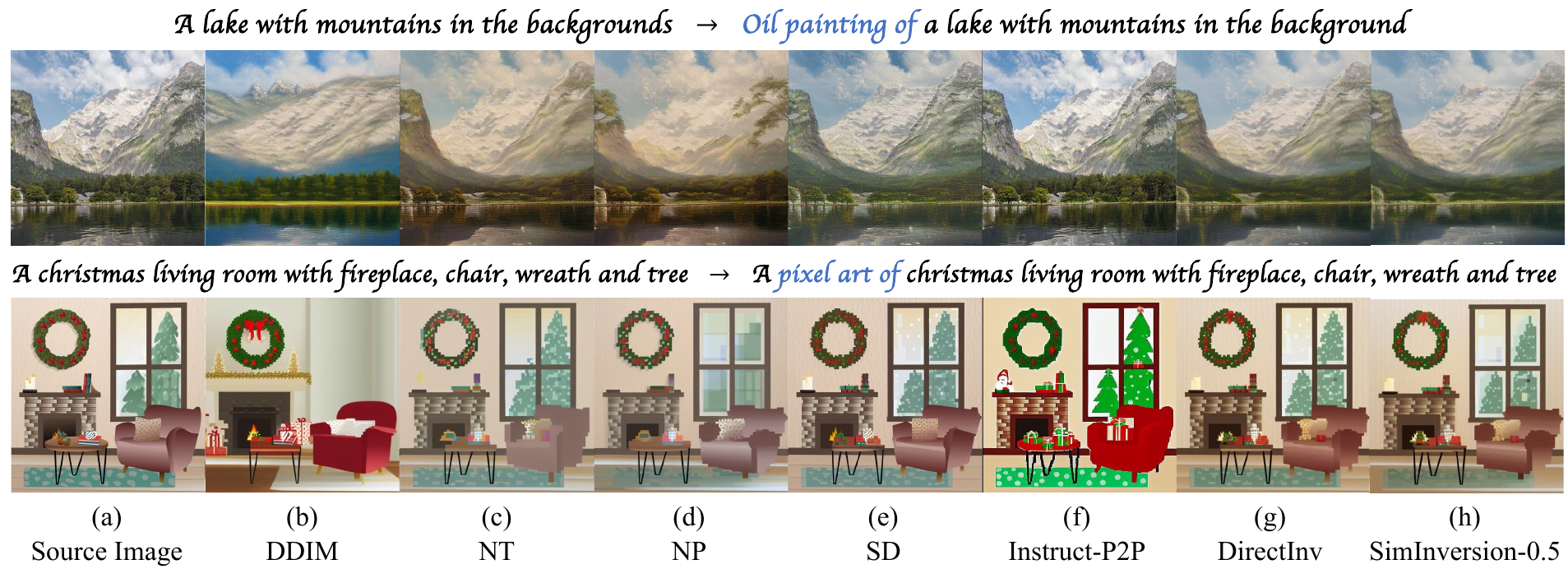}
\captionof{figure}{Illustration of image editing for changing style.}\label{fig:type9}
\end{minipage}
\end{figure}



\end{document}